\documentclass{article}

\usepackage[preprint]{nips_2018}

\usepackage[utf8]{inputenc} 
\usepackage[T1]{fontenc}    
\usepackage{hyperref}       
\usepackage{url}            
\usepackage{booktabs}       
\usepackage{amsfonts}       
\usepackage{nicefrac}       
\usepackage{microtype}      
\usepackage{algorithm}
\usepackage{algpseudocode}
\usepackage{graphics}
\usepackage{graphicx}
\usepackage{subfigure}
\usepackage[english]{babel}
\usepackage{amsthm}
\usepackage{amssymb}
\title{Cortex Neural Network: learning with Neural Network groups}
\author{
Liyao Gao\\
Department of Computer Science\\
Purdue University\\
\texttt{gao463@purdue.edu} \\
}

\begin{document}

\maketitle

\begin{abstract}
  Neural Network has been successfully applied to many real-world problems, such as image recognition and machine translation. However, for the current architecture of neural networks, it is hard to perform complex cognitive tasks, for example, to process the image and audio inputs together. Cortex, as an important architecture in the brain, is important for animals to perform the complex cognitive task. We view the architecture of Cortex in the brain as a missing part in the design of the current artificial neural network.

  In this paper, we purpose Cortex Neural Network (CrtxNN). The Cortex Neural Network is an upper architecture of neural networks which motivated from cerebral cortex in the brain to handle different tasks in the same learning system. It is able to identify different tasks and solve them with different methods. In our implementation, the Cortex Neural Network is able to process different cognitive tasks and perform reflection to get a higher accuracy. We provide a series of experiments to examine the capability of the cortex architecture on traditional neural networks. Our experiments proved its ability on the Cortex Neural Network can reach accuracy by 98.32\% on MNIST and 62\% on CIFAR10 at the same time, which can promisingly reduce the loss by 40\%. 
  

\end{abstract}

\section{Introduction}
Researchers have been focused a lot on neural networks. In recent years, a series of neural networks have been introduced. Many of the neural networks can reach a satisfying performance in simple cognitive tasks [1].
Examples would be the convolutional neural network in the field of image object recognition, recurrent neural network in speech recognition, and LSTM in machine translation. However, currently, it is still difficult for a neural network to solve complex cognitive tasks [2, 10]. A complex cognitive task describes a more difficult task which requires cognitive process compared to a simple cognitive task. In this paper, it includes multi-cognitive task processing and learning with reflection. 

Firstly, multi-cognitive task processing is nearly impossible. Using a single deep neural network to process image, audio and video together might lead a miserable performance and require a large amount of training data. In the current research, the Multitask learning mainly focused on an approach to inductive transfer that improves learning for one task by using the information contained in the training signals of other related tasks [13], like transfer learning. Also, in a learning process, based on the current architecture, we miss the part of reflection. Human will perform a reflection when approaching its limit in performance [11]. 
The current neural network architecture cannot perform reflection. For solving the complex cognition tasks, a key is missing here. 

\begin{figure}
\centering
\includegraphics[width=9.3cm]{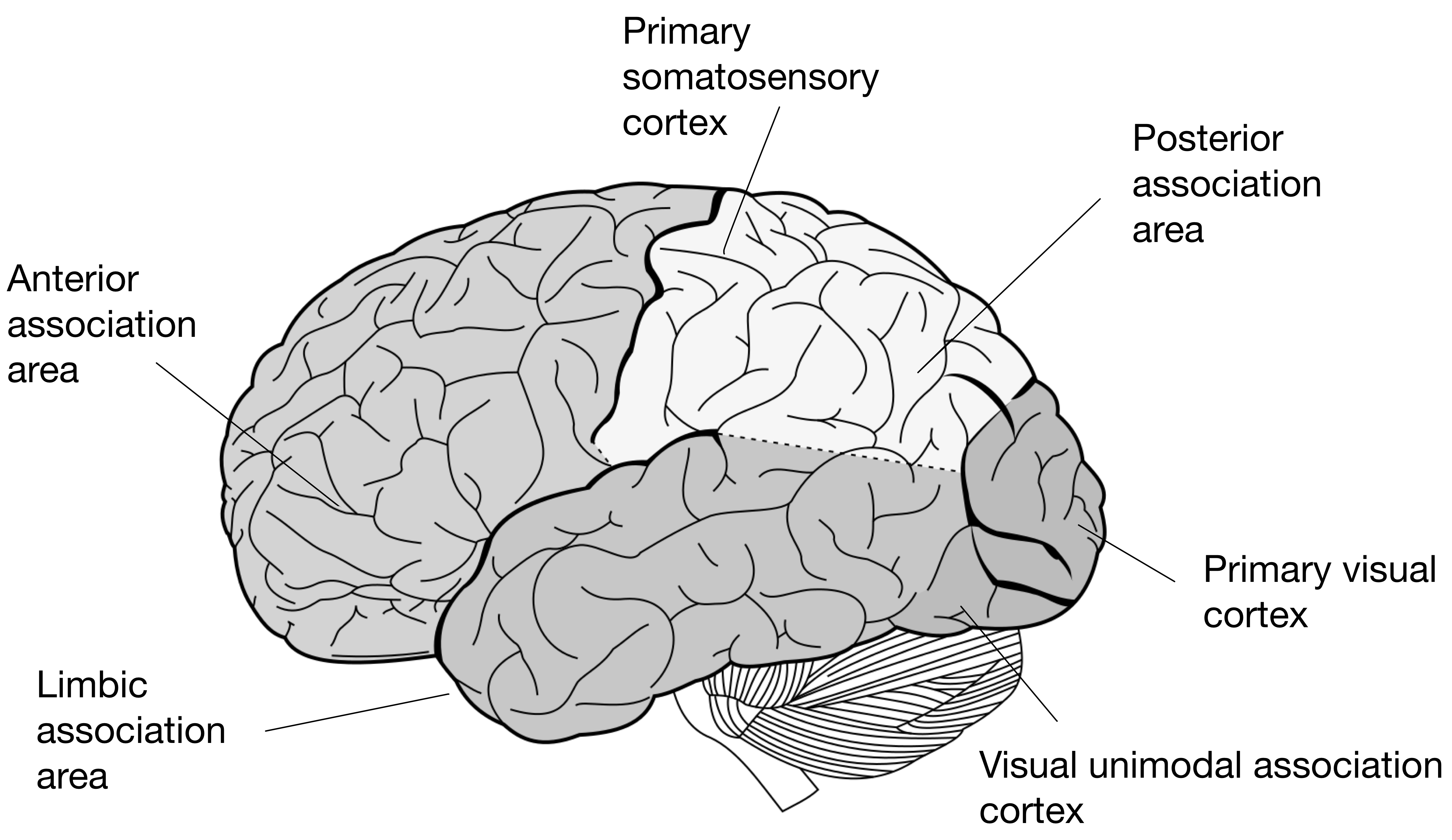}
\caption{Schematic representation of cortex in brain[7]}
\end{figure}

In the architecture of the brain, biological neural networks are not the only structure that is important for processing of intellectual activities. In the cerebral cortex, the sensory area will recognize different inputs based on sensations and then different neural networks located in association cortex will be active to process different tasks (in Figure 1). In fact, the architecture of cortex is necessary for the brain to handle complex cognitive tasks [5, 7]. Since it is necessary for the human brain to let its biological neural network to perform complex cognition task with the architecture of cortex, for the Artificial Neural Network, as a simulated logical architecture of the biological neural network, should also require the architecture of cortex to perform complex cognition tasks [3].


In this paper, we would like to introduce the Cortex Neural Network (CrtxNN). CrtxNN is an architecture that can empower artificial neural network to process complex cognitive tasks. It is a system that learns with neural network groups based on the architecture cerebral cortex. We view this as the missing key for Artificial Neural Networks to solve complex cognitive tasks. In our implementation, the CrtxNN is able to solve the complex cognitive tasks: multi-cognition tasks learning and reflection, in following methods: 

\textbf{1. Multi-cognition tasks learning and processing}. Multi-cognition task learning and processing is important for a powerful learning system. In our architecture, CrtxNN is possible to learn multiple different cognitive tasks with mixed datasets. The CrtxNN will separate task by different sensations, such as image, audio and video, and train a series of neural networks for different tasks. After the learning phase, it is able to recognize different tasks and to process them with a different solution using the corresponding neural network. 

\textbf{2. Reflection}. Reflection is a complex cognitive task and an important learning process [11]. In the CrtxNN, a single neural network will be initially trained for a task, which will be viewed as a general situation processor. Then, a series of neural networks will be trained on exceptional situations which are the parts where the general situation processor has a bad performance. Normally, after training and reflecting, the CrtxNN will be able to perform a strategy decision: to decide which network to be used based on the incoming data. This is able to get a higher performance for a single neural network when a single neural network is approaching its performance limit. 

We provided a series of experiments. The CrtxNN was built on typical neural networks. We compared the CrtxNN with the typical neural network. We tested the result on two tasks: approximation of two-dimensional functions, and mixed image datasets. Our experiment reached a satisfying result. For the approximation of two-dimensional function, we reduced the loss at most by 99\%. For the mixed image datasets, out model can process MNIST and CIFAR10 at the same time and lower the loss by 40\% in average. 

Our contribution of this paper can be summarized as the following points: 

\begin{enumerate}
  \item Using the architecture of Cortex in the brain as a method for artificial neural networks to solve complex cognitive tasks. 
  \item Purpose the learning theory and algorithm of Cortex Neural Network to provide a solution for multi-cognitive tasks learning and reflection.
  \item Design and conduct experiments to test the feasibility of the architecture and theory. The experiment result showed the positive effect of Cortex Neural Network.
\end{enumerate}

\section{Cortex Neural Network Architecture}
\subsection{Model Architecture}

Our design of CrtxNN is learned from the architecture of Cortex in the human brain. The cerebral cortex is very important for brain to perform complex-cognitive process, such as perception, attention, cognition and so on [7]. When a human is performing a cognitive task, the brain is using the corresponding neural network which locates in the association cortex areas to process the tasks [8].The cortex is divided into different functional areas for modular functionality [7]. Basically, the cortex is combined by different neural networks located in association areas that are specified on different tasks. For example, the visual cortex is active when looking at a cat but when the brain tries to understand a stop sign, both the language and visual association area will be active to process the image and understand the sign. 

There is three important structure in the cerebral cortex that we use in CrtxNN. 
\begin{figure}
\centering
\centering
\includegraphics[width=9.3cm]{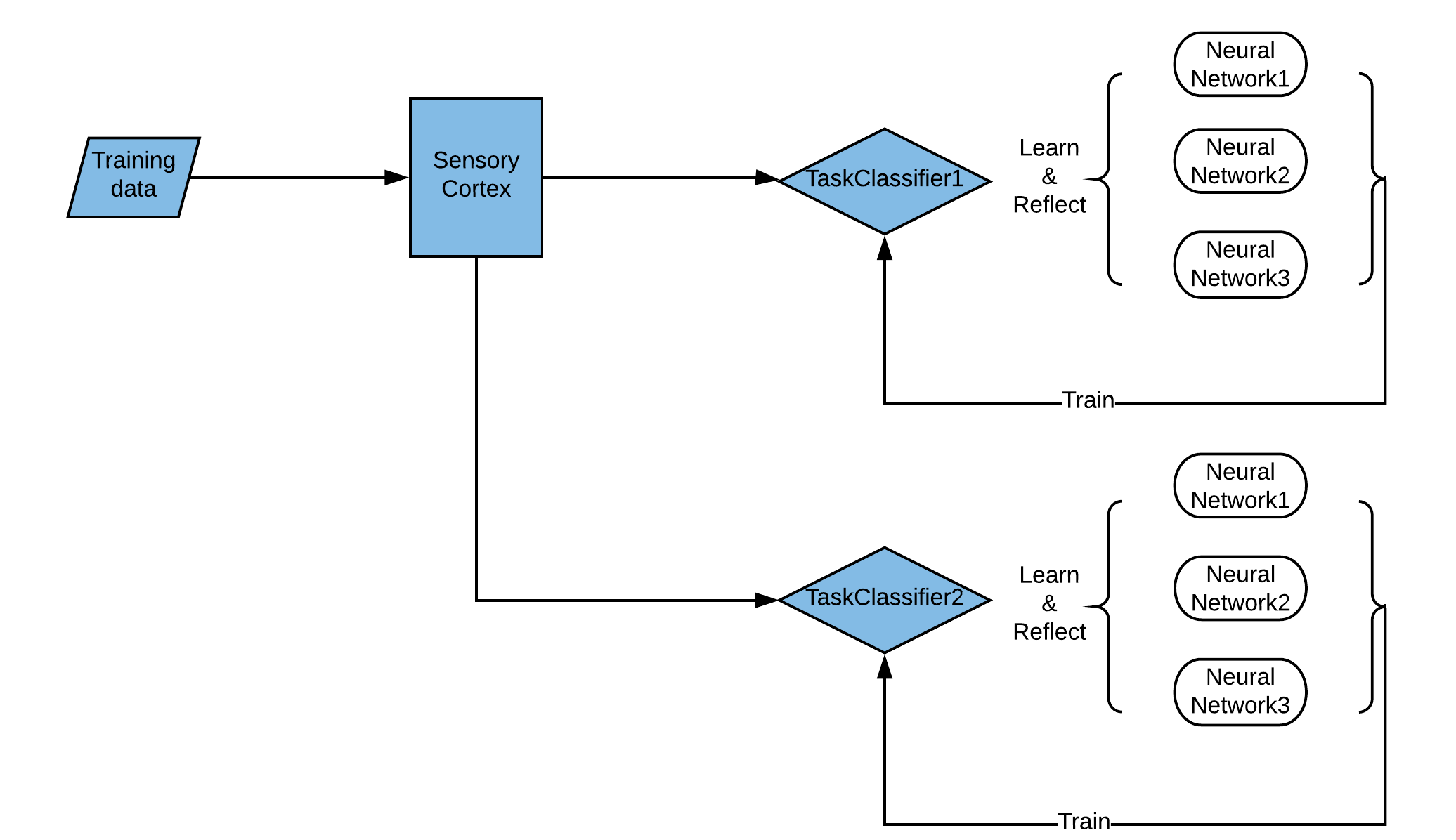}
\caption{CrtxNN's architecture.}
\end{figure}

\begin{itemize}
  \item Sensory area, or the primary somatosensory cortex. For human beings, at the primary somatosensory cortex, the perception areas are orderly arranged from the toe to mouth. [14]
  \item Association area. They enable us to interact differently to the world, and support logical thinking, understanding and language. The association areas are combined by distributed networks in cortex. [8]
  \item Biological neural network. The artificial neural network was designed based on the biological neural network. 
\end{itemize}

In the CrtxNN, we are trying to model the Cortex in the human brain to allow complex cognition tasks to be performed by Artificial Neural Networks. As it is shown in Figure 2, we extracted several important structures for our implementation. Sensory cortex receives all sensory input from the body and sends information to responding network to handle the signal.

\subsection{Model Description}

\textbf{Sensory cortex}
This is the sensory phase of CrtxNN which plays the same role as the Primary somatosensory cortex. The sensory cortex is designed to classify the data by senses. It is a simulation of the way that human brain is used to separate tasks by different sensory organs, such as vision, hearing and smell. In our case, we classify the data by its shape of input and output tensor to represent different senses [future work]. The sensory cortex will initialize a new task classifier during the learning process and will assign data to the corresponding task classifier during the predicting process. 

The input Set $S={x_1, x_2, x_3, ..., x_n}$, the sensory cortex can produce a series of set $X_1, X_2, ..., X_n$ which $\forall m, \forall a,b \in X_m, shape(x_a)=shape(x_b)$, and $X_1\cup X_2 \cup ...\cup X_n = S$. The sets $X_1, X_2, ..., X_n$ will be separately trained by a series of base neural networks.  

\textbf{Association cortex area}
The association cortex area contains a set of base neural networks and task classifiers(see Figure 3). It is the place to process the task based on recognized types. It is used to learn different tasks, reflect to get better performance and, finally, predict the input data using the series of base neural networks. The learning and reflection algorithm will be discussed in the next part.

\textbf{Task classifier}
There will be a new Task Classifier when the Sensory Cortex find a data with different dimension appears. A task classifier represents a meta-task in CrtxNN. After reflection, it is used for the Cortex to decide which is the best network to use when predicting the label of the input data. The core of the task classifier is a decision tree. 

The input Set $X = \{x_{1}, x_{2}, x_{3}, ..., x_{n}\}$, output Set $Y = \{y_{1}, y_{2}, y_{3}, ..., y_{n}\}$. The output set $Y$ is defined by reflection, it is labeled by the id of the corresponding neural network. 

The decision tree is using a number $\theta$ as a threshold for splitting the data into different clusters. The method that we are using for evaluating the performance of a expand is Gran Ratio. The Gran Ratio is defined by the entropy of the split data and the entropy of the clusters. 

The entropy of a discrete set can be descried as: $H(P) = -\sum_{i=1}^{n} P(i)log(P(i))$. The Gran Ratio Method: $Gran\;Ratio = \frac{H(Y)-\sum_{V\in Values} \frac{|Y_v|}{|Y|} H(X_v)}{-\sum_{V\in Values} \frac{|Y_v|}{|Y|}log\frac{|Y_v|}{|Y|}}$

\textbf{Base neural network}
The base neural network is the neural network we choose for the CrtxNN to use in all tasks. For example, a LeNet-5 can be used as a base neural network in the task of handwritten recognition, MNIST. The experiment part provides more examples. 

\subsection{Performance}
\textbf{Reflection}
Let $\delta$ be the possibility of a correct event, $P(C) = \delta, P(W) = (1 - \delta)$. The MSE loss of a reflected neural network group is $L_{MSE} = \frac{1}{n}\sum_{n=1}^{n} (Y_{i} - \overline{Y})^2$, we have: 

$$L_R = \frac{1}{\delta n}\sum_{i=1}^{\delta n} (c_{i} - y_{i})^2 + \frac{1}{(1 - \delta) n}\sum_{i=1}^{(1 - \delta) n} (w_{i} - y_{i})^2$$
\textbf{Multi-cognitive tasks}
Suppose that in the CrtxNN, there exists $k$ reflected neural network group. The MSE loss of the CrtxNN is:

$$L_{CrtxNN} = \frac{\sum_{i=1}^{k} L_{Ri}}{k}$$

\section{Learning \& Predicting Process}

\subsection{Learning Process}
In a learning process, the CrtxNN goes through the following phases: sensing, training and reflecting. At the beginning, the input data will go through the sensing part. The sensory cortex will assign the input data by different dimension into different task classifiers. Then, the data grouped by different task classifiers will start to train by a base neural network, which normally, is the general situation. After that, the CrtxNN will begin to reflect by collecting part of data with high error to train with some new base networks. The task classifiers will be retrained after this process.


\begin{algorithm}
\caption{CrtxNN learning algorithm}
\begin{algorithmic}[1]
\Procedure {TrainCrtxNN}{$X$, $Y$}
\State $S(X, Y) = SensoryCortex(X, Y)$
\ForAll {$x, y \in S(X, Y)$}
\State $NetworkSubSet = \{\}$
\State $network = TrainBaseNetwork(x, y)$
\State $NetworkSubSet.append(network)$
\ForAll {$x, y \in S(X, Y)$}
\State $ErrSet.add(y-NN.predict(x))$
\EndFor
\State $NetworkReflectionSet = Reflect(X, Y, ErrSet)$
\State $NetworkSubSet.append(NetworkReflectionSet)$
\EndFor
\EndProcedure

\end{algorithmic}
\end{algorithm}

\subsection{Reflection}
As a complex cognition task, reflection is important in learning. Humans frequently perform a reflection when they are approaching their extreme in performance when learning. Normally, human will try to analyze the problem, divide them into different cases, and try to apply different methods to solve them separately. There are a number of reflection methods and one of the commonest methods is to reflect from the mistakes. In CrtxNN, we are trying to apply this reflection method. 

We suppose the training input is Set $X = \{x_{1}, x_{2}, x_{3}, ..., x_{n}\}$, and training output is Set $Y = \{y_{1}, y_{2}, y_{3}, ..., y_{n}\}$. The hypothesis Set, which is prediction by a base neural network, is $H = \{h_{1}, h_{2}, h_{3}, ..., h_{n}\}$. 
The prediction is contained by Correct Set $C = \{c_{1}, c_{2}, c_{3}, ..., c_{n}\}$and Wrong Set $W = \{w_{1}, w_{2}, w_{3}, ..., w_{n}\}$. Furthermore, $H = C+W$.

When $|h_{x}-y_{x}| < \epsilon$, an event of correct happens. $Err_{max}$ is the largest error in set W. Otherwise, an event of wrong happens. The reflection method will separate W set into $k$ clusters based on its distance, using the Kmeans clustering algorithm, which aims at minimizing the squared error function: $J = \sum_{i=1}^{k} \sum_{j=1}^{n} (||x_i-v_j||)^2$, where $||x_i-v_j||$ is the Eucledian distance between a point. Then, the data which split by the k different clusters will be trained by different neural networks. 

\subsubsection{Reflection Bound}
\newtheorem{theorem}{Theorem}
\begin{theorem}
Based on the method of reflection above, where $N$ is sufficiently large, $\exists \; \epsilon, k \in [1, +\infty)$, to make $L_{CrtxNN}<L_{NN}$. The largest prediction error is $Err_{max} = \max_{i}\; |w_i-y_i|$. We suppose $Err_{max} = t\epsilon, t>1$, $[\epsilon, Err_{max}]$ describes a reflection bound. 
\end{theorem}

\begin{proof}As we stated before, according to the definition, the loss of CrtxNN can be described as:
$$L_{RelfNN} = \frac{1}{\delta n}\sum_{i=1}^{\delta n} (c_{i} - y_{i})^2 + \frac{1}{(1 - \delta) n}\sum_{i=1}^{(1 - \delta) n} (w_{i} - y_{i})^2$$

\newtheorem{claim}{Claim}

\begin{claim}
With adequate training process, the loss of a neural network is lower than predicting with the Weighted Arithmetic Average of the data. $\forall w_i \in W, $ $|w_i-y_i| \leqslant C = \frac{\sum_{i=1}^{(1-\delta) n} x_{w_{i}}w_{i}}{\sum_{i=1}^{(1-\delta) n} x_{w_{i}}}$
\end{claim}

Based on the claim, the estimated loss of a reflected neural network can be rewritten into the following form: 
$L_{RelfNN} \leqslant \frac{1}{\delta n}\sum_{i=1}^{\delta n} (c_{i} - y_{i})^2 + \frac{1}{(1 - \delta) n}\sum_{i=1}^{(1 - \delta) n} C^2$.

Let $(1-\delta)n = N$. The expected reduce in loss of CrtxNN, $R= \sum_{i=1}^{N}(w_{i} - y_{i})^2 - NC^2$

\newtheorem{assumption}{Assumption}

\begin{assumption}
We suppose here that set W is an Uniform Distribution in $Err_{max}$ is the largest error in set W. 
\end{assumption}

The best expected reduce in loss of CrtxNN can be rewrite into: 

$$R = \max_{t, k, N}\sum_{i=1}^{N}(1+\frac{t-1}{N}i)^2 - N(\frac{t-1}{k})^2$$

When $k=t-1, N>2$, $R>0$. Furthermore, when $N$ is sufficiently large, when $k > 1$ and $\frac{4}{3k^2}<\frac{(86-t)t-73}{36(t-1)^2}$, $R>0$. In both cases, $L_{CrtxNN}<L_{NN}$. Therefore, the theorem is proved. 
\end{proof}




\subsubsection{Task Classifier Training}
The task classifier plays an important role in reflection. It is served as a strategy decider or a task understanding unit in the CrtxNN. We decide to use a decision tree classifier here is since we have a reflection bound for every group of data. For a continuous function, the decision tree uses a threshold to classify the data. The task classifier receives the grouped training data which is the result of reflection. It will be trained by the input data and corresponding neural network id. According to the reflection bound, the input data groups will be bounded into different ranges. This is the last step of a reflection method. It could be able to handle by a decision tree with a high performance.
The decision of $k$ affects the performance of task classifier. In practice, we usually take $k=2$.

\subsection{Predicting Process}
After the learning and reflecting process, the association area of CrtxNN will evolve into an architecture with several task classifiers to handle different tasks and a series of neural networks to process it. When the system is trying to process an input data, the sensory cortex will first determine which part of association area it belongs to. After this step, the input data will be delivered to the corresponding task classifier. The task classifier will decide the exact neural network for prediction of the data. Finally, the CrtxNN will output the result of that network. This is the way of CrtxNN to simulate the way that cortex will perform when processing a cognitive task in the brain. 




\section{Experiments}
We applied the CrtxNN in two situations with different networks. The first experiment is a focusing on two-dimensional space. The second experiment can be viewed as a proof that the reflection can work in a high dimensional space. These two experiments provide strong evidence of the CrtxNN's ability to perform multi-cognitive tasks and reflection. The result can lead a reduction in loss by 99\% in function approximation and 40\% in mixed datasets on MNIST and CIFAR10. 

\subsection{Approximation of function in two dimensional space}
In this experiment, we are trying to test the ability of CrtxNN based on a task of approximation of function in two-dimensional space. The base neural network we choose here is a feed forward network. We implemented the network and make CrtxNN approximate the assigned functions based on learning data. We tested the result on linear continuous function, higher-order continuous function, linear discontinuous function and higher-order discontinuous function. The base network we choose here is a simple Feed Forward Network with one hidden layer using RELU as activation function. The Adam is used here as the optimizer with the learning rate of $0.0001$. The discontinuous function is trying to simulate the situation that different tasks are involved in the training data.

\begin{figure}[h]
  \centering
  \includegraphics[scale=0.2]{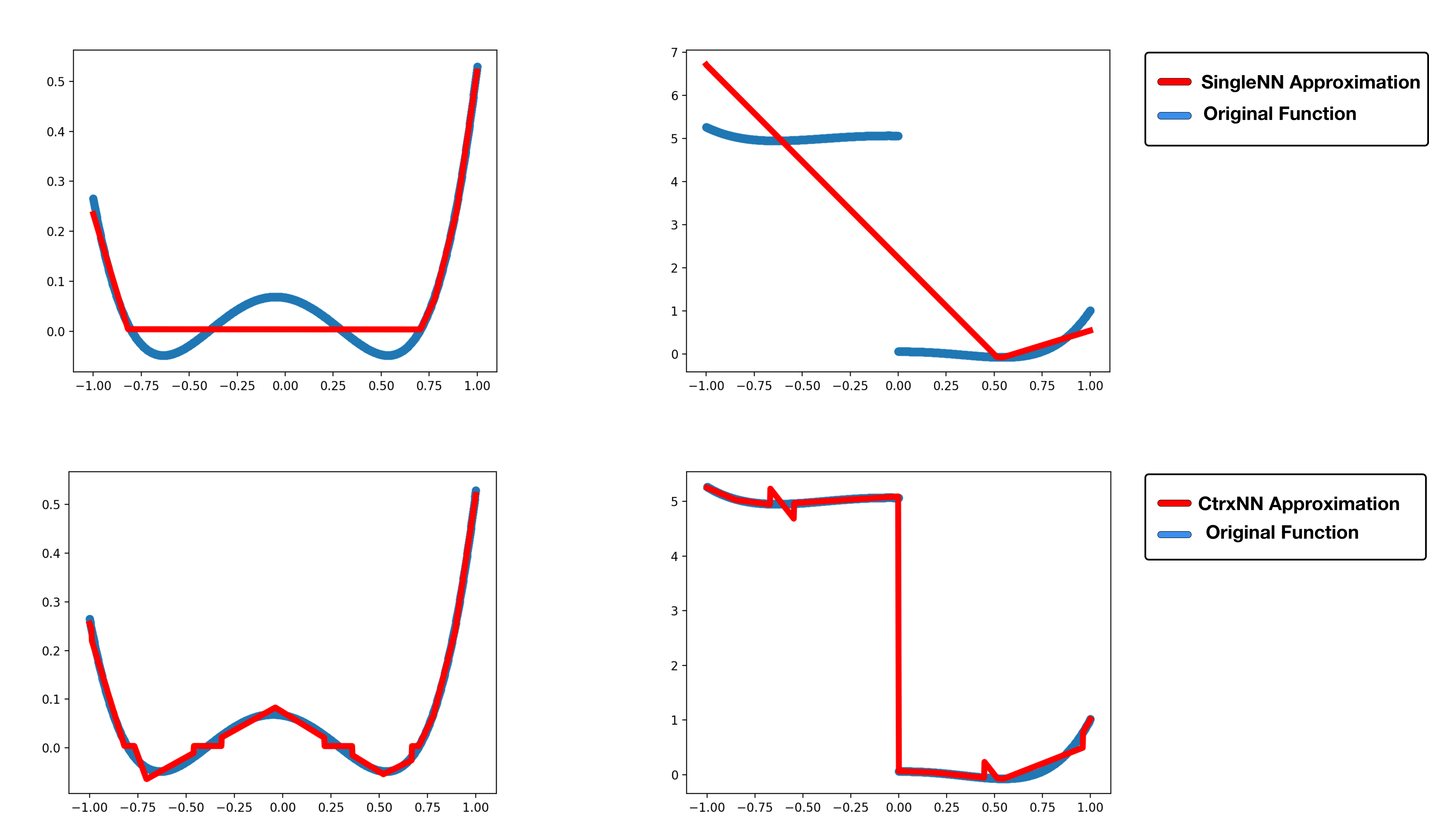}
  \caption{Performance of Single Feed Forward Network on approximation task between before and after applying CrtxNN}
\end{figure}

\begin{tabular}{ |p{5.5cm}||p{2.2cm}|p{2.2cm}|p{2.2cm}|  }
 \hline
 \multicolumn{4}{|c|}{Preformance of CtrxNN and SingleNN} \\
 \hline
 Function&BaseNN Loss&CtrxNN Loss&Loss reduce \%\\
 \hline
 $f(x) = x$   & $3\times10^{-9}$    &$3\times10^{-10}$&$90\%$\\
  \hline
 $f(x)=x^3-0.2x-0.35$   &   $4\times10^{-4}$  & $2\times 10^{-4}$&$92\%$\\
  \hline
 $f(x) = x^4 + 0.2 x^3 - 0.67 x^2$&$1\times10^{-3}$ & $5\times10^{-5}$&$95\%$\\
  \hline
  $f(x) = x^4 + x^3 - 0.6 x^2, x\in [-1,0)$
&  &  &\\
 $f(x) = x^5 +  x^4 - 0.5 x^3,    x\in [0,1)$& $1.4$  & $1.4\times10^{-2}$&$99\%$\\
   \hline

\end{tabular}

Based on the result, we can find out that the CtrxNN is able to have a better performance compared to the base neural network (SingleNN). An interesting observation here is: as the function is more complex, the base network became worse in performance. However, the reduction in loss for CtrxNN is becoming higher. Especially in the case of the discontinuous function, it can be approximated with a much higher performance. This experiment show explains how CtrxNN works (see Figure 3) and provide strong evidence on the situation that the CrtxNN can have a much better performance as the function is becoming more complex. 

\subsection{Mixed MNIST, CIFAR10 datasets}
This experiment could be a real situation that is used for the CrtxNN. We created a MNIST \& CIFAR10 mixed datasets simply putting the two datasets together. We design this experiment for three purposes: 1. to see CrtxNN's ability in multi-cognitive tasks 2. to provide another evidence of the performance of CrtxNN's reflection; 3. based on the Experiment of function in two-dimensional space, to show that the CrtxNN is able to handle a real-world image task. 

In this experiment, we mix the MNIST datasets and CIFAR10 datasets together as a new learning datasets. The base neural network we choose here is a CNN with two convolutional layers, two linear layers, two pooling and RELU as activation function. Adam is used here as the optimizer with the learning rate of 0.00001. Our expected result here is to see the CrtxNN can handle different tasks after learning and have an increase in performance after reflecting. 

After the training process, the CrtxNN produced two taskClassifier and six base neural networks in total. The result is showed in following picture and table. 

\begin{figure}[h]
  \centering
  \includegraphics[scale=0.6]{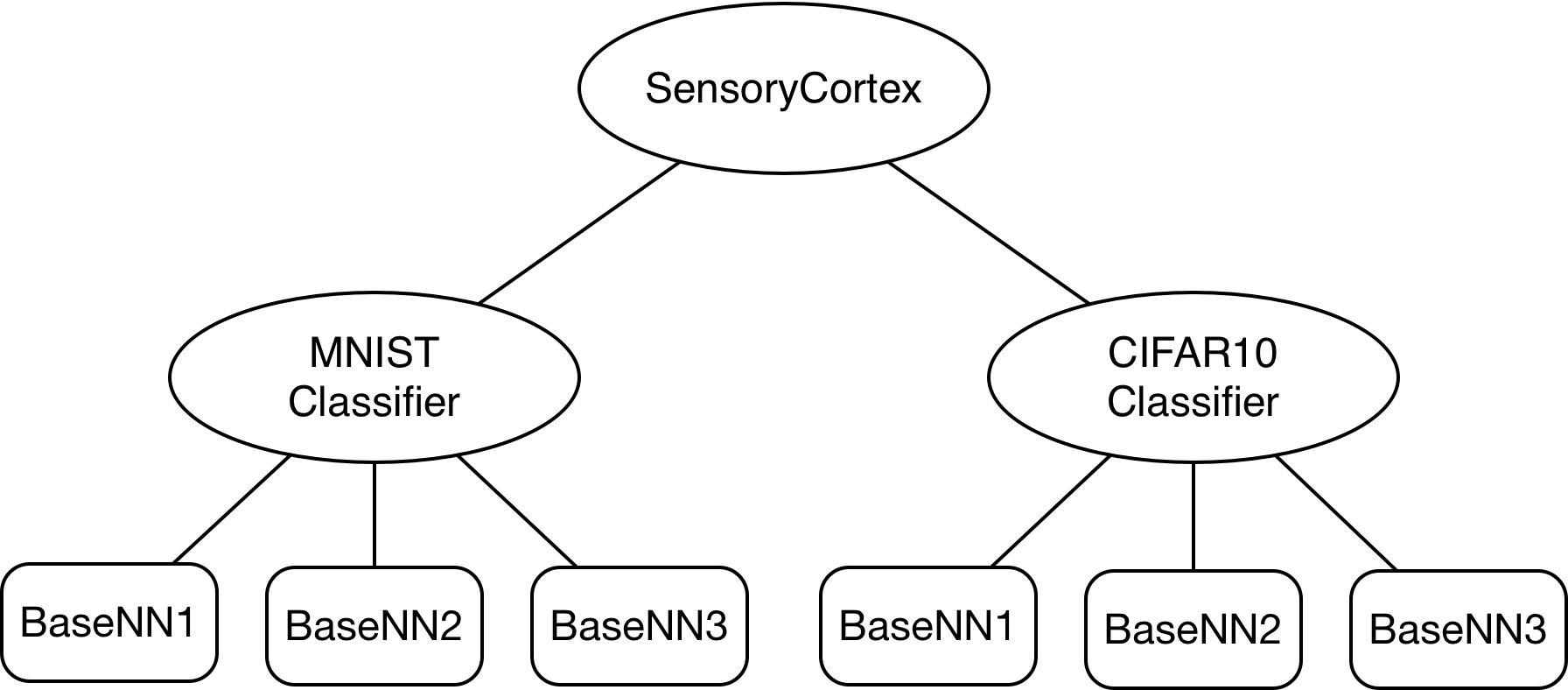}
  \caption{The architecture of CrtxNN after training by MNIST \& CIFAR10 mixed datasets}
\end{figure}



\begin{tabular}{ |p{3.5cm}||p{2cm}|p{2cm}|p{2cm}|p{2cm}|  }
 \hline
 \multicolumn{5}{|c|}{Preformance of CtrxNN and SingleNN} \\
 \hline
 Datasets&MNIST\_NN Accuracy&CIFAR10\_NN Accuracy&CtrxNN \space\space\space\space\space\space Accuracy&Loss reduce \%\\
 \hline
 MNIST   & $98$\%    & X&$98.32\%$&$41.7\%$\\
  \hline
 CIFAR10   &  X & $52\%$&$60\%$&$26.7\%$\\
  \hline

\end{tabular}

Bbased on the result of this experiment, we can make a conclusion that the CrtxNN is able to do multi-cognition tasks and the reflection can have an improvement in performance on image tasks. At this point, we can infer that this architecture can allow more different tasks to be processed and have a better performance. The CrtxNN should be able to solve images, audio, video tasks at the same time. For the multi-cognition tasks, the CrtxNN is able to separate MNIST and CIFAR10 datasets and train corresponding neural network on different tasks. For the reflection logic, according to the result, the CrtxNN generated a series of base neural networks to handle the different situation and the loss has been reduced by over 40\%.



\section{Related works}
\textbf{Information Bottleneck} 

The information bottleneck theory might provide an important clue for the CrtxNN in design. As what has been described in this theory, the neural network is able to perform learning tasks by shrinking the input information and try to make the easiest connection to the output [4]. This might be able to explain why CrtxNN's idea of using a series of neural networks can be better than trying to use a single neural network for solving the complex cognition tasks. The self-information can be really high at the function's part with high derivative. This would require an exponentially larger amount lot of training data and epoch. Therefore, the information that needs the neural network to shrink can be really high. In the case of multi-cognitive tasks, the data might need to be reshaped into the same dimensional shape. This would create many points that might have a high derivative. A CrtxNN might be able to solve it as a linear expansion. 

\textbf{Multi-task reinforcement learning} [6]

We consider the CtrxNN as a similar architecture compared to multi-task reinforcement learning. We believe that this is a biological applaudable method of learning. In this work, the researchers proved its possible to use a smaller number of tasks to reach a high performance. We used the concept and expanded the work in our implementation. 

\section{Discussion}
\textbf{Limitations}

Our implementation of CrtxNN is only one possible method to approach Cortex architecture. Currently, we found its ability to handle multi-cognitive tasks and reflection is obviously superior to the traditional architecture of using a single neural network. In this paper, we only discover two possible complex cognitive tasks that CrtxNN might be able to solve. We strongly believe that it can handle more complex cognitive tasks, such as concept learning, metacognition and critical thinking. 

In our implementation, we simply use the shape of data to be the make the decision in sensory cortex. In fact, there can be better methods to perform a better task. Also, for the part of the base neural network, currently, in our implementation, we only tried to use the different base network to learn the different task. This approach is not elegant. A network that is able to change and fit the input data is expected here. Unfortunately, it is not fully supported by the current research. Scientists are working on it recently [12] but it still takes time. In short, we view this as an easy architecture for the industry to simplify the training process and an adjustable base network is always welcomed to join the architecture. 

\textbf{Theoretical result} 

We may believe that with the power of the architecture of Cortex, the neural network can reduce the loss in any network to be 0. Because the reflection method might always be able to reduce the loss by certain percentage, if we can make the CrtxNN even deeper, there could result in a further reduction in loss. This will be carefully tested as a future work. 

Also, we can possibly infer from the current result that the CrtxNN will reach the best performance currently in all deep learning tasks today. The reflection method is always able to reduce the loss no matter what base network we have. Therefore, if the base neural network of CrtxNN is the network that currently has the best performance, it will be better for loss and accuracy. 

\textbf{Future work} 

We will focus on using the CrtxNN as a way to approach the performance which its limit is 1. More detailed mathematical and computational proofs will be studied. Currently, we believe that the characteristic of the neural network on discontinuous function and the information bottleneck theory may provide clues. Also, we only focused on dividing the problems with the new neural network in this paper. However, in a real learning process, the pruning is also considered. To merge the similar tasks is important. PackNet [9] is a possible method to merge. 

\section{Conclusion}

In this paper, we performed the CrtxNN as a bionic approach motivated by cortex in the human brain to solve complex cognition tasks based on Artificial Neural Network. Our method, CrtxNN, enables the Artificial Neural Networks to perform multi-cognition tasks promisingly and to reflect on errors to continuously get a better performance. Based on the result of our experiment, the CrtxNN could perform multi-cognition tasks and can reduce the loss by 40\% in MNIST and CIFAR10.

\section*{References}

\small

[1] LeCun, Y., Bengio, Y., \& Hinton, G. (2015). Deep learning. \textit{nature}, 521(7553), 436.

[2] Anderson, J. R. (1996). ACT: A simple theory of complex cognition. \textit{American Psychologist}, 51(4), 355.

[3] Margulies, D. S., \& Smallwood, J. (2017). Converging evidence for the role of transmodal cortex in cognition. \textit{Proceedings of the National Academy of Sciences}, 201717374.

[4] Shwartz-Ziv, R., \& Tishby, N. (2017). Opening the black box of deep neural networks via information. \textit{arXiv preprint arXiv:1703.00810}.

[5] Koechlin, E., Basso, G., Pietrini, P., Panzer, S., \& Grafman, J. (1999). The role of the anterior prefrontal cortex in human cognition. \textit{Nature}, 399(6732), 148.

[6] Wilson, A., Fern, A., Ray, S., \& Tadepalli, P. (2007, June). Multi-task reinforcement learning: a hierarchical Bayesian approach. \textit{In Proceedings of the 24th international conference on Machine learning} (pp. 1015-1022). ACM.

[7] Siegelbaum, S. A., \& Hudspeth, A. J. (2000). \textit{Principles of neural science} (Vol. 4, pp. 1227-1246). E. R. Kandel, J. H. Schwartz, \& T. M. Jessell (Eds.). New York: McGraw-hill.

[8] Thomas Yeo, B. T., Krienen, F. M., Sepulcre, J., Sabuncu, M. R., Lashkari, D., Hollinshead, M., ... \& Fischl, B. (2011). The organization of the human cerebral cortex estimated by intrinsic functional connectivity. \textit{Journal of neurophysiology}, 106(3), 1125-1165.

[9] Mallya, A., \& Lazebnik, S. (2017). PackNet: Adding Multiple Tasks to a Single Network by Iterative Pruning. \textit{arXiv preprint arXiv:1711.05769}.

[10] Knauff, M., \& Wolf, A. G. (2010). Complex cognition: the science of human reasoning, problem-solving, and decision-making.

[11] Boud, D., Keogh, R., \& Walker, D. (Eds.). (2013). \textit{Reflection: Turning experience into learning}. Routledge.

[12] Kaiser, L., Gomez, A. N., Shazeer, N., Vaswani, A., Parmar, N., Jones, L., \& Uszkoreit, J. (2017). One model to learn them all. \textit{arXiv preprint arXiv:1706.05137}.

[13] Caruana, R. (1998). Multitask learning. In \textit{Learning to learn} (pp. 95-133). Springer, Boston, MA.

[14] Cauller, L. (1995). Layer I of primary sensory neocortex: where top-down converges upon bottom-up. \textit{Behavioural brain research}, 71(1-2), 163-170.
















\end{document}